\documentclass{article}

\usepackage[preprint]{neurips_2026}

\usepackage[utf8]{inputenc} 
\usepackage[T1]{fontenc}    
\usepackage{hyperref}       
\usepackage{url}            
\usepackage{booktabs}       
\usepackage{amsfonts}       
\usepackage{nicefrac}       
\usepackage{microtype}      
\usepackage{xcolor}         
\usepackage{amsmath}
\usepackage{amssymb}
\usepackage{amsthm}
\usepackage{graphicx}
\usepackage{algorithm}
\usepackage{algorithmic}
\usepackage{multirow}
\usepackage{mathtools}
\usepackage{enumitem}
\usepackage[capitalize,noabbrev]{cleveref}

\title{Equilibrium Residuals Expose Three Regimes of Matrix-Game\\Strategic Reasoning in Language Models}

\author{%
\textbf{Wenhua Nie \quad BINHAN LUO \quad Zijie Meng}\\
\textbf{Jyh-Shing Roger Jang \quad Ching-Wen Ma}\\[3pt]
\normalfont Correspondence: Wenhua Nie, National Taiwan University\\
\normalfont \texttt{d13944014@ntu.edu.tw}
}

\newtheorem{theorem}{Theorem}

\newtheorem{definition}{Definition}

\newcommand{\RR}{\mathbb{R}}
\newcommand{\exploit}{\mathrm{Exploit}}

\begin{document}

\maketitle

\begin{abstract}
Large language models can score well on named game-theory benchmarks while failing on the same strategic computation once semantic cues are removed.
We show this gap with procedurally generated zero-sum matrix games: a model that recognizes familiar games drops to 34\%, 18\%, and 2\% success on anonymous $2{\times}2$, $3{\times}3$, and $5{\times}5$ payoff matrices.
The benchmark separates semantic recall, learned approximate Nash computation, and an output-interface bottleneck that limits scale.
Training only on $2{\times}2$ and $3{\times}3$ games, supervised fine-tuning raises unseen $5{\times}5$--$7{\times}7$ success from 2\% to 61\%, while exploitability-reward training averages 37\% with high seed variance.
We prove that the exploitability residual is $2$-Lipschitz in payoff perturbations, unlike discontinuous vertex-returning LP equilibrium selectors, explaining why residual training can transfer under payoff shifts even when formatting instability limits mean performance.
A dominated-action padding experiment provides causal evidence: trained models solve $3{\times}3$ games embedded in much larger matrices, while random-padded controls fail and dense $12{\times}12$ games remain near failure.
Procedural evaluation is therefore necessary for measuring strategic reasoning, and residual rewards expose a real but format-limited route to approximate equilibrium computation.

\end{abstract}

\section{Introduction}
\label{sec:intro}

Can large language models reason strategically in matrix games, or do they merely recall textbook solutions?
Recent evaluations report that frontier LLMs, when given chain-of-thought prompting,
select Nash-equilibrium actions in classic games such as the Prisoner's Dilemma,
Chicken, and Battle of the Sexes with high
accuracy~\citep{gandhi2023strategic}, fueling claims that language models possess
genuine game-theoretic competence. Systematic analyses further probe LLM
rationality in economic games, finding persistent gaps in formal
strategic reasoning~\citep{fan2024can}. However, these benchmarks share a critical confound:
every test instance is a \emph{named} game whose equilibrium has appeared thousands of times
in the training corpus.

We demonstrate that this apparent competence is fragile. When we strip away the semantic
cues and present anonymous procedural payoff matrices, performance
drops to 34\%, 18\%, and 2\% on anonymous random
$2{\times}2$, $3{\times}3$, and $5{\times}5$ matrices (\cref{fig:hero}, left).
Few-shot prompting with solved examples does not help---it actually \emph{hurts}, reducing success from 34\% to 10\% on
$2{\times}2$ games, suggesting that examples trigger template-matching rather than algorithmic
reasoning.

This memorization--computation gap raises a fundamental question: \emph{can LLMs learn to
actually compute equilibria, rather than merely recall them?} We investigate this using
procedurally generated matrix games as a controlled microscope. Each training step presents
a freshly sampled payoff matrix, eliminating the possibility of memorization. We study two
learning paradigms:

\begin{itemize}[nosep,leftmargin=*]
\item \textbf{SFT} (Supervised Fine-Tuning): the model is trained on solver-labeled
  $({\rm matrix}, {\rm Nash~strategy})$ pairs computed by linear programming.
\item \textbf{VERGE} (Verifiable Equilibrium-Regret GRPO): the model receives only an
  exploitability reward---a scalar certificate measuring distance from Nash equilibrium---with
  no access to the equilibrium itself.
\end{itemize}

Both methods, trained exclusively on $2{\times}2$ and $3{\times}3$ games, produce models that
generalize to unseen $5{\times}5$--$7{\times}7$ games, improving success rates (exploitability
$< 0.10$) from 2\% to 61\% (SFT) and 37\% on average for VERGE (minimum seed 14\%;
higher variance but no solver labels needed). Our goal is not to claim that VERGE is the
best in-distribution solver---SFT with oracle labels is stronger there---but to use the
contrast between label imitation and residual rewards as a diagnostic for what kind of
computation transfers. We therefore treat VERGE as a mechanism probe rather than a
deployment-ready solver: the central evidence is the separation among semantic lookup,
approximate computation, and output-interface failure. Both methods outperform the maximin
pure-strategy heuristic at $5{\times}5$ (28\%) and uniform mixing (4\%).

A deeper analysis reveals three distinct regimes in this matrix-game setting, each governed by
different mechanisms:

\paragraph{Regime~I: Semantic Lookup.} The base model retrieves memorized name$\to$action
mappings. This regime achieves perfect accuracy on textbook games but degrades sharply on
procedural instances, and is disrupted rather than aided by few-shot examples.

\paragraph{Regime~II: Learned Approximate Computation.} Fine-tuned models exhibit genuine
algorithmic behavior: they satisfy game-theoretic invariances (permutation equivariance error
$< 0.06$; payoff shift/scale invariance error $< 0.06$) and generalize across game sizes.
Under a near-OOD payoff-distribution shift (integer$\to$Gaussian), a single VERGE checkpoint
substantially outperforms SFT, but 100-game multi-seed replication shows that this transfer is
seed- and format-sensitive rather than a stable mean advantage. This is consistent with our
theoretical result that exploitability residuals are Lipschitz-stable in the payoff matrix
while LP equilibrium selectors are discontinuous at support boundaries, but the theorem is
only a mechanism probe, not a learned-policy generalization bound.

\paragraph{Regime~III: Output-Interface Failure.} Both methods exhibit a sharp performance
cliff at $n{\approx}12$. We provide evidence ruling out token length alone:
$3{\times}3$ games padded to $12{\times}12$
token counts retain 83\% success, while true $12{\times}12$ games achieve only 7\%. Embedding
$3{\times}3$ games into $12{\times}12$ matrices via iteratively dominated padding yields
high success for trained models (SFT $95{\pm}8\%$ over three seeds; VERGE $90\%$ in the
diagnostic seed), while random-padded
controls fail (\cref{fig:hero}, right), leaving strategic depth and structured
serialization as the remaining bottlenecks. A low-rank curriculum tradeoff reinforces this finding:
training on larger games ($3$--$5$) improves conditional strategy quality but sharply degrades output-format generalization
(valid rate at $12{\times}12$ drops from 99\% to 0\%).

We further prove (\cref{thm:residual_stability}) that the exploitability residual is
$2$-Lipschitz in payoff perturbations, while vertex-returning LP equilibrium selectors are
discontinuous. This theorem provides a mechanism consistent with the observed pattern:
residual rewards can be more stable under payoff-distribution shift, while label imitation
can still dominate when exact solver labels match the test distribution.

\begin{figure}[t]
\centering
\includegraphics[width=\textwidth]{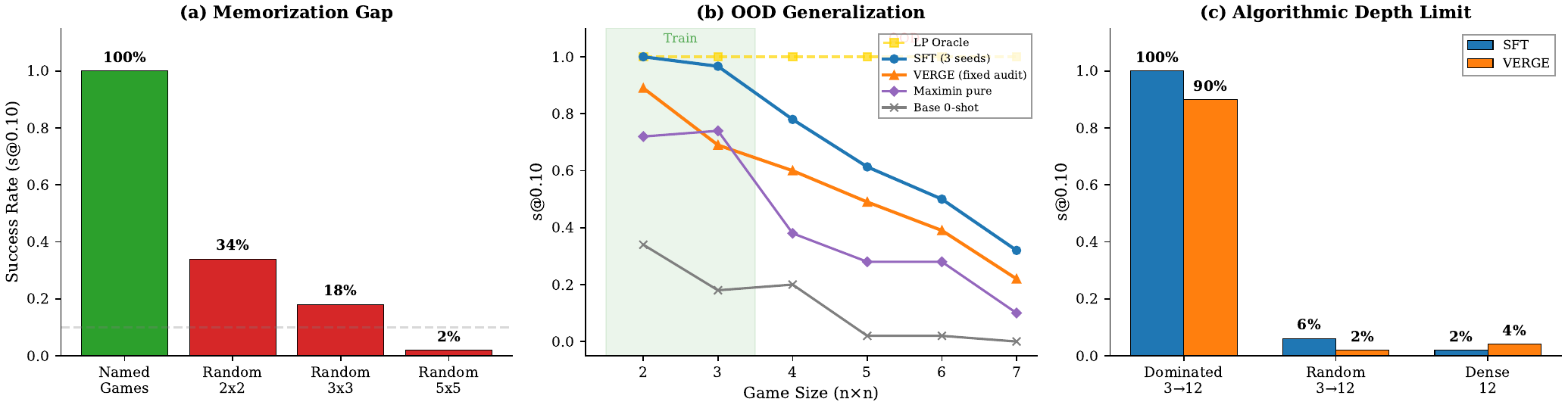}
\caption{\textbf{Three regimes of matrix-game strategic reasoning in LLMs.}
\textbf{(a)}~Base model performs well on named games but falls to 34\%, 18\%, and 2\% on random 2$\times$2, 3$\times$3, and 5$\times$5 matrices (memorization gap).
\textbf{(b)}~SFT and VERGE trained on 2--3$\times$3 generalize to 7$\times$7 OOD, far exceeding maximin and base.
\textbf{(c)}~$3{\times}3$ games embedded in $12{\times}12$ matrices succeed for both SFT and VERGE, while dense and random-padded $12{\times}12$ games fail---a depth/serialization bottleneck in this setting rather than a token-length artifact.}
\label{fig:hero}
\end{figure}

\paragraph{Contributions.} Our main contributions are:
\begin{enumerate}[nosep,leftmargin=*]
\item A causal audit showing that LLM game-theoretic competence on named matrix-game benchmarks
  reflects semantic recall, not computation, with procedural and invariance-based evidence
  (\cref{sec:memorization}).
\item A controlled training comparison showing that solver-supervised SFT is stronger
  in-distribution, while residual-reward VERGE does not achieve a stable mean transfer
  advantage over SFT but can produce high-variance payoff-shift checkpoints
  (\cref{sec:methods,sec:transfer}).
\item Controlled evidence ruling out token length alone as the explanation for the
  $n{\approx}12$ cliff under our LoRA-adapted matrix-game setting, leaving
  strategic depth and serialization as the bottleneck candidates
  (\cref{sec:depth_limit}).
\item A formal gradient-cancellation result proving that role-merged GRPO in zero-sum self-play
  yields zero advantage-weighted strategic gradient, motivating the cooperative exploitability
  formulation (\cref{sec:proof_grpo}).
\end{enumerate}

\section{Related Work}
\label{sec:related}

\paragraph{Strategic reasoning evaluation in LLMs.}
A growing body of work evaluates LLM behavior in strategic settings.
\citet{gandhi2023strategic} show that chain-of-thought prompting enables LLMs to generalize strategic behavior across classic games; related prompting work shows that explicit reasoning traces can improve algorithmic tasks~\citep{wei2022chain,kojima2022large,nye2021show,zelikman2022star}.
\citet{duan2024gtbench} introduce GTBench, showing that even frontier models fail on many
game-theoretic tasks. \citet{fan2024can} study whether LLMs can serve as rational agents
in economic games. \citet{silva2024mixed} evaluate LLMs on mixed-strategy Nash games and
find that performance degrades under even slight game modifications. TMGBench
\citep{wang2024tmgbench} improves coverage over $2{\times}2$ game types and story
contexts, but remains an evaluation benchmark rather than a training and mechanistic
analysis framework. Richer strategic agents such as Cicero combine language with planning in
Diplomacy~\citep{bakhtin2022human}; our matrix setting is narrower but gives exact certificates.
These works share a critical limitation: they primarily evaluate
named or small structured games whose equilibria can leak through the training corpus.
Our procedural-generation approach eliminates this confound by testing on freshly sampled
payoff matrices, revealing that apparent strategic competence is largely memorization.

\paragraph{Algorithmic reasoning and length generalization in transformers.}
Transformers have been studied as computational devices that can learn algorithms from
examples~\citep{garg2022can}. Work on length generalization shows that models trained on
short sequences often fail on longer ones, with the failure mode depending on positional
encoding and task structure~\citep{anil2022exploring}. Our dominated-padding experiment
connects to this literature by separating two failure modes: output-length limitations
(which we rule out) and algorithmic complexity (which remains as the bottleneck candidate
at $n \approx 12$).

\paragraph{Equilibrium computation.}
Classical computational game theory studies Nash equilibria and algorithms for matrix games
and beyond~\citep{nash1951non,lemke1964equilibrium,nisan2007algorithmic}. We use the
zero-sum LP-solvable case because it gives exact exploitability certificates for evaluation.

\paragraph{Reinforcement learning from verifiable rewards.}
RLVR has emerged as an alternative to human preference learning, using mathematical
verifiers~\citep{shao2024deepseekmath}, code execution~\citep{le2022coderl}, or formal
proofs as reward signals. Our exploitability reward is a game-theoretic instance of this
paradigm: it provides a scalar certificate of solution quality without revealing the
solution itself, analogous in spirit to verifier-based reward modeling for math,
step-level verification, and preference learning~\citep{cobbe2021training,lightman2023lets,christiano2017deep,guo2025deepseekr1}. Our multi-seed
replication shows verifier rewards can produce high-transfer checkpoints but not stable
mean dominance, making exploitability a diagnostic rather than proof of RL superiority.

\paragraph{Self-play and multi-agent RL for LLMs.}
Self-play has been applied to LLMs for preference optimization~\citep{wu2024sppo},
iterative fine-tuning~\citep{chen2024self}, and general improvement~\citep{yuan2024self}.
MARSHAL~\citep{yuan2025marshal} trains strategic LLM agents through multi-agent self-play
in cooperative and competitive games, showing gains on held-out games and reasoning tasks.
\citet{munos2024nash} frame preference optimization as Nash equilibrium finding.
In multi-agent RL, Policy-Space Response Oracles (PSRO)~\citep{lanctot2017unified}
provide a principled framework for computing Nash equilibria through iterative best-response
computation.
These works use ``self-play'' loosely---the model competes against itself for quality
improvement, long-horizon interaction, or preference optimization. We study the classical
matrix-game sense with exact exploitability certificates, and we prove an implementation-specific
gradient-cancellation result (\cref{sec:proof_grpo}) showing that role-merged GRPO normalization
removes the advantage-weighted strategic gradient in this setting, motivating our cooperative
exploitability formulation instead.

\paragraph{GRPO and policy gradient methods.}
GRPO~\citep{shao2024deepseekmath} replaces the PPO critic with group-relative normalization,
showing strong results on mathematical reasoning. It has since been widely adopted for LLM
post-training. We prove that role-merged GRPO normalization in zero-sum self-play
yields zero advantage-weighted strategic gradient (\cref{sec:proof_grpo}), motivating our cooperative
exploitability formulation where each response is independently scored.

\section{Method and Theory}
\label{sec:method}

\subsection{Procedural Game Testbed}
\label{sec:memorization}

We study two-player zero-sum matrix games defined by a payoff matrix
$A \in \RR^{n \times n}$, where the row player receives $A_{ij}$ and the column player
receives $-A_{ij}$ when row~$i$ and column~$j$ are played.
A \emph{mixed strategy} is a probability distribution over actions:
$\mathbf{p} \in \Delta^n$ for the row player, $\mathbf{q} \in \Delta^n$ for the column player.

\begin{definition}[Exploitability / NashConv]
\label{def:exploitability}
For strategies $(\mathbf{p}, \mathbf{q})$ in a zero-sum game with payoff matrix~$A$,
the \emph{exploitability} is:
\begin{equation}
\exploit(A, \mathbf{p}, \mathbf{q}) = \underbrace{\max_i (A\mathbf{q})_i - \mathbf{p}^\top A \mathbf{q}}_{\text{row regret}} + \underbrace{\mathbf{p}^\top A \mathbf{q} - \min_j (\mathbf{p}^\top A)_j}_{\text{column regret}}.
\label{eq:exploitability}
\end{equation}
$\exploit = 0$ if and only if $(\mathbf{p}, \mathbf{q})$ is a Nash equilibrium.
We normalize by the payoff range: $\bar{\exploit} = \exploit / (2(\max A - \min A))$, and
define the \emph{exploitability reward} as $r = 1 - \bar{\exploit} \in [0, 1]$.\footnote{When
$\max A = \min A$ (constant matrix), we perturb one entry by $+1$ to avoid division by zero;
see \cref{sec:exp_details}.}
\end{definition}

\paragraph{Game generation.}
Each training step samples a fresh $n \times n$ matrix with integer entries from $[-9, 9]$,
normalized to $[-2, 2]$ by the payoff range. Training uses $n \in \{2, 3\}$; evaluation
extends to $n \in \{2, \ldots, 7\}$ (OOD) and $n \in \{8, \ldots, 20\}$ (far OOD).
This procedural generation eliminates memorization: the model must analyze each matrix
from scratch.

\paragraph{Prompt and output format.}
The model receives the matrix in numerical form and outputs a JSON object
\texttt{\{"row":[...],"col":[...]\}} specifying mixed strategies for both players.
Strategies are normalized to the probability simplex after parsing.

\paragraph{Cooperative exploitability formulation.}
In both training methods (SFT and VERGE), the model proposes a complete
strategy pair $(\mathbf{p}, \mathbf{q})$ per turn, receiving the exploitability reward
$r = 1 - \bar{\exploit}(A, \mathbf{p}, \mathbf{q})$. This is a single-agent cooperative
formulation: each response is scored independently on how close its proposed strategies
are to a Nash equilibrium. We note that a natural alternative---role-merged self-play
where each response is scored as both row and column player---leads to gradient
cancellation (\cref{sec:proof_grpo}). Our cooperative formulation avoids this
pathology by construction.

\subsection{Two Teaching Methods}
\label{sec:methods}

\paragraph{SFT (Supervised Fine-Tuning).}
For each generated game, we compute the Nash equilibrium using linear programming
(the minimax theorem guarantees existence in zero-sum games). The model is trained with
cross-entropy loss on (matrix, equilibrium) pairs, using LoRA ($r=32$, $\alpha=64$) on
all attention and MLP projections.

\paragraph{VERGE (Verifiable Equilibrium-Regret GRPO).}
The model generates strategy proposals and receives the exploitability reward
$r = 1 - \bar{\exploit}$ as feedback. No Nash equilibrium labels are provided.
GRPO~\citep{shao2024deepseekmath} generates $G=8$ samples per game and computes
group-relative advantages over the accumulated training batch, with clipped policy
gradient updates and KL regularization against the reference policy.

\paragraph{Theoretical comparison.}
SFT and VERGE optimize fundamentally different objectives. SFT imitates a specific
equilibrium selector $T(A) = (\mathbf{p}^*, \mathbf{q}^*)$ determined by the LP solver.
VERGE minimizes the exploitability residual $\exploit(A, \mathbf{p}, \mathbf{q})$.
The following elementary design lemma characterizes their distinct stability properties
and motivates the payoff-shift diagnostics; it is not a learned-policy generalization
bound:

\begin{theorem}[Residual stability vs.\ selector instability]
\label{thm:residual_stability}
For any fixed strategies $(\mathbf{p}, \mathbf{q}) \in \Delta^n \times \Delta^n$:
\begin{equation}
|\exploit(A, \mathbf{p}, \mathbf{q}) - \exploit(B, \mathbf{p}, \mathbf{q})| \leq 2 \|A - B\|_\infty.
\end{equation}
That is, $\exploit(\cdot, \mathbf{p}, \mathbf{q})$ is Lipschitz in the payoff matrix.

Conversely, for any $\delta > 0$, there exist games $A, A'$ with
$\|A - A'\|_\infty < \delta$ such that $\|T(A) - T(A')\|_1 = \Omega(1)$,
where $T$ is any equilibrium selector that returns a vertex solution at degeneracy
(as common deterministic LP tie-breaking rules do).
\end{theorem}

\begin{proof}[Proof sketch]
Expanding~\eqref{eq:exploitability}, the $\mathbf{p}^\top A \mathbf{q}$ terms cancel,
giving $\exploit(A, \mathbf{p}, \mathbf{q}) = \max_i (A\mathbf{q})_i - \min_j (\mathbf{p}^\top A)_j$.
The Lipschitz bound follows from
$|\max_i (A\mathbf{q})_i - \max_i (B\mathbf{q})_i| \leq \|A - B\|_\infty$
(since $\|\mathbf{q}\|_1 = 1$) and the analogous bound for the $\min_j$ term,
where $\|A - B\|_\infty = \max_{i,j} |A_{ij} - B_{ij}|$ is the max-entry norm.
Selector instability follows from degeneracy: for the scaled game
$A_\epsilon = \epsilon \cdot M$ where $M$ has a unique interior equilibrium,
$T(A_\epsilon) \to T(M)$ as $\epsilon \to 0^+$, but at $\epsilon = 0$ the all-zeros
matrix admits every strategy as a Nash equilibrium, and a vertex-returning tie breaker returns a vertex,
causing an $\Omega(1)$ jump. Full proof in~\cref{sec:proof_stability}.
\end{proof}

\paragraph{Consequence.} For any fixed strategy pair, VERGE's verification signal is
smooth under payoff perturbations, whereas a deterministic solver label can jump at
support-boundary degeneracies. This does not prove learned-model OOD generalization; it
motivates the selected payoff-shift tests in~\cref{sec:transfer}.

\paragraph{Why not zero-sum self-play?}
A natural alternative to cooperative exploitability training is competitive self-play,
where each response is scored as both row and column player with antisymmetric rewards
$(r_i, -r_i)$. We prove in~\cref{sec:proof_grpo} that \emph{role-merged} GRPO---where
the same generated output $y_k$ is evaluated in both roles and both rewards enter the
same normalization group---yields zero advantage-weighted strategic gradient, excluding
any KL term that only pulls toward the reference policy. This specific failure mode
motivates our cooperative exploitability formulation, where GRPO with sufficient
gradient accumulation
trains stably (see~\cref{sec:exp_details} for training details).

\section{Experiments}
\label{sec:experiments}

All experiments use Qwen3.5-9B with LoRA ($r=32$, $\alpha=64$) on
A800-80GB GPUs. SFT and VERGE train for 2000 steps with learning rates
$5{\times}10^{-5}$ and $10^{-5}$. Evaluation uses 50 games per size and four
samples per game at temperature 0.7. We report $s@0.10$: the fraction of games
where the \emph{best-of-4} sample achieves $\bar{\exploit}<0.10$. This uses
oracle exploitability selection and therefore measures the learned hypothesis
space, not deployment-ready decoding; the main OOD table also reports pass@1.
All evaluation games use seed 99. The uniform row in \cref{tab:main_ood}
calibrates a non-adaptive mixed-strategy null: 4\% success at both
$5{\times}5$ and $7{\times}7$. For $N{=}50$, the worst-case binomial standard
error is 7.1 percentage points; the distribution-transfer table uses $N{=}30$
diagnostic games per condition (9.1 points).

\paragraph{Evidence package.}
The experiments ask whether $2{\times}2$--$3{\times}3$ training transfers to larger
anonymous games, whether residual rewards behave differently from solver labels under
payoff shift, whether the $n{\approx}12$ cliff is caused by length, dimension, or
strategic depth, whether higher-rank adapters remove it, and whether zero-sum self-play
needs the cooperative exploitability objective. The positive claims are restricted to
these controlled settings; we do not claim that VERGE globally dominates SFT.

\subsection{Out-of-Distribution Generalization}
\label{sec:ood}

\begin{table}[t]
\centering
\caption{OOD generalization on zero-sum games (train 2--3, test 2--7). All methods
use the same evaluation games (seed=99, 50 games per size). $s@0.10$ is best-of-4
success with exploitability selection; pass@1 uses only the first sampled output.
Fine-tuned rows report mean$\pm$std over 3 independent training seeds; dashes mark
unevaluated few-shot sizes.}
\label{tab:main_ood}
\vspace{0.5em}
\small
\begin{tabular}{lcccccc}
\toprule
\textbf{Method} & \textbf{2$\times$2} & \textbf{3$\times$3} & \textbf{4$\times$4} & \textbf{5$\times$5} & \textbf{6$\times$6} & \textbf{7$\times$7} \\
\midrule
Base (0-shot) & 0.34 & 0.18 & 0.20 & 0.02 & 0.02 & 0.00 \\
Few-shot (3 ex.) & 0.10 & 0.02 & 0.00 & 0.00 & -- & -- \\
Qwen3-32B (8\% valid) & 0.00 & 0.04 & 0.02 & 0.00 & 0.02 & 0.00 \\
Uniform & 0.18 & 0.04 & 0.16 & 0.04 & 0.08 & 0.04 \\
Maximin (pure) & 0.72 & 0.74 & 0.38 & 0.28 & 0.28 & 0.10 \\
\midrule
SFT $s@0.10$ (3 seeds) & \textbf{1.00}$\pm$.00 & \textbf{0.97}$\pm$.04 & \textbf{0.78}$\pm$.00 & \textbf{0.61}$\pm$.09 & \textbf{0.50}$\pm$.05 & \textbf{0.32}$\pm$.09 \\
SFT pass@1 & \textbf{1.00}$\pm$.00 & \textbf{0.84}$\pm$.09 & 0.49$\pm$.01 & 0.33$\pm$.07 & 0.23$\pm$.03 & 0.12$\pm$.03 \\
Llama-3-8B SFT $s@0.10$ & -- & -- & -- & \textbf{0.61} & -- & 0.18 \\
Llama-3-8B SFT pass@1 & -- & -- & -- & 0.36 & -- & 0.07 \\
VERGE $s@0.10$ (3 seeds) & 0.85$\pm$.12 & 0.66$\pm$.22 & 0.50$\pm$.24 & 0.37$\pm$.22 & 0.29$\pm$.26 & 0.18$\pm$.09 \\
VERGE pass@1 & 0.77$\pm$.21 & 0.50$\pm$.26 & 0.31$\pm$.20 & 0.21$\pm$.18 & 0.11$\pm$.10 & 0.07$\pm$.03 \\
\midrule
LP Oracle & 1.00 & 1.00 & 1.00 & 1.00 & 1.00 & 1.00 \\
\bottomrule
\end{tabular}
\end{table}

\cref{tab:main_ood} shows the core result. Both SFT and VERGE improve upon the base model:
$5 \times 5$ success rises from 2\% (base) to 61\% (SFT) and 37\% (VERGE; seeds 58\%,
40\%, and 14\%), versus 28\% for maximin pure strategy.
A 150-game replication at $5{\times}5$/$7{\times}7$ gives SFT
$0.63{\pm}.01$/$0.28{\pm}.05$ and VERGE $0.42{\pm}.18$/$0.16{\pm}.15$,
confirming the OOD pattern while preserving VERGE's higher seed sensitivity.
Thus VERGE's primary diagnostic signal is not absolute performance---where SFT with oracle
labels dominates---but selected distribution transfer (\cref{sec:transfer}) and the absence
of solver supervision.
Few-shot prompting \emph{decreases} performance (34\%$\to$10\% at $2 \times 2$), suggesting
template-matching interference rather than induced computation.
Under this strict JSON-only output contract, a zero-shot Qwen3-32B format-stress
baseline obtains mean $s@0.10=0.01$ and valid-output rate 0.08 across $2$--$7$;
we therefore do not interpret it as an isolated test of reasoning capacity.

Without oracle selection, SFT pass@1 remains far above the base model: 0.33$\pm$0.07 at
$5 \times 5$ and 0.12$\pm$0.03 at $7 \times 7$, versus 0.00 for the base model at both
sizes. VERGE pass@1 is 0.21$\pm$0.18 at $5 \times 5$ and 0.07$\pm$0.03 at
$7 \times 7$. Best-of-4 therefore strengthens the headline numbers, but the pass@1 rows
show that the OOD signal does not rely solely on oracle selection.

\subsection{Distribution Transfer}
\label{sec:transfer}

\begin{table}[t]
\centering
\caption{Single-seed diagnostic distribution transfer (train: integer $[-9,9]$; test:
Gaussian, sparse, integer). $s@0.10$ on $8 \times 8$ to $20 \times 20$ OOD games,
30 games per condition for this table. The 100-game Gaussian replication uses
3 seeds for SFT/final VERGE and 4 partially overlapping seeds for VERGE-600.
The uniform
baseline becomes dominant at $n \geq 12$ on Gaussian/sparse due to random-matrix
concentration; the learned models' advantage is at smaller, structured sizes.}
\label{tab:dist_transfer}
\vspace{0.5em}
\small
\begin{tabular}{llccccc}
\toprule
\textbf{Distribution} & \textbf{Method} & \textbf{8$\times$8} & \textbf{10$\times$10} & \textbf{12$\times$12} & \textbf{15$\times$15} & \textbf{20$\times$20} \\
\midrule
\multirow{3}{*}{Gaussian} & Uniform & 0.33 & 0.47 & 0.63 & 0.93 & \textbf{1.00} \\
& SFT & 0.50 & 0.30 & 0.40 & 0.20 & 0.37 \\
& VERGE & \textbf{0.83} & \textbf{0.60} & 0.33 & 0.27 & 0.00 \\
\midrule
\multirow{3}{*}{Sparse} & Uniform & 0.83 & 1.00 & \textbf{1.00} & \textbf{1.00} & \textbf{1.00} \\
& SFT & \textbf{1.00} & 0.97 & 0.80 & 0.70 & 0.37 \\
& VERGE & 1.00 & 0.93 & 0.63 & 0.47 & 0.07 \\
\midrule
\multirow{3}{*}{Integer} & Uniform & 0.03 & 0.03 & 0.00 & 0.10 & 0.10 \\
& SFT & 0.10 & 0.07 & 0.00 & 0.00 & 0.00 \\
& VERGE & \textbf{0.23} & \textbf{0.07} & 0.03 & 0.00 & 0.00 \\
\bottomrule
\end{tabular}
\end{table}

\cref{tab:dist_transfer} reveals a clear but scoped pattern. The single-seed diagnostic
contains high-transfer VERGE cells on Gaussian payoffs, but a 100-game replication changes
the conclusion from ``mean advantage'' to ``high-variance mechanism probe.'' Across three
training seeds, SFT is stable on Gaussian $8{\times}8$ and $10{\times}10$
($0.49{\pm}.02$, $0.44{\pm}.08$). Final VERGE is lower and much higher variance
($0.36{\pm}.29$, $0.28{\pm}.30$). VERGE-600 over 4 partially overlapping
seeds reaches $0.45{\pm}.22$/$0.53{\pm}.34$ at $8{\times}8$/$10{\times}10$,
but the shared seeds move in opposite directions, so we treat this as checkpoint
variability rather than a controlled early-vs-late gain. Thus residual rewards can
produce high-transfer checkpoints, but output-format stability controls whether that
potential appears in the final policy.

Importantly, the uniform (maximin) baseline becomes increasingly competitive on
Gaussian and sparse distributions at large $n$: it reaches $s@0.10 = 1.00$ on
Gaussian $20 \times 20$ and sparse $12$--$20$. This occurs because large random
matrices concentrate toward games where uniform is near-optimal. The learned
models' advantage is confined to $n \leq 10$--$12$, before random-matrix concentration
makes uniform mixing near-optimal. We do not interpret the small non-monotonic changes
within a row (e.g., SFT on Gaussian $15$--$20$) as scaling laws under this 30-game
diagnostic evaluation.

On sparse games at $8 \times 8$, SFT and VERGE both reach $s@0.10=1.00$, but SFT's
sample-level valid-output rate remains higher (100\% vs.\ 70\%). At larger sparse sizes
this format gap widens (70\% vs.\ 2.5\% valid at $20 \times 20$), so when the bottleneck
is format reliability rather than strategy quality, supervised label training has an
advantage.
Thus VERGE's transfer advantage should be read as conditional on producing parseable
strategies, not as a claim that residual rewards solve the output-interface problem.
Its $0.00$ on Gaussian $20 \times 20$ is therefore primarily a format-validity failure
(valid rate 0\%), not evidence that a parseable VERGE strategy is exploitable in every
such game.
On integer payoffs (the training distribution), both learned methods and the uniform
baseline perform poorly at $n \geq 12$, consistent with the depth/serialization bottleneck
identified in~\cref{sec:depth_limit}.

\subsection{The $n \approx 12$ Depth/Serialization Bottleneck}
\label{sec:depth_limit}

Both methods show a sharp performance cliff near $n = 12$. We design two controls
to test whether this reflects algorithmic complexity or output-length limitations.

\paragraph{Padding control.}
We pad $3 \times 3$ game prompts with irrelevant system-note text to match the exact token
count of $12 \times 12$ prompts (798 tokens). If the cliff were caused by token length or
positional encoding degradation, padded $3 \times 3$ games should also fail.

\paragraph{Dominated-action padding with random-padding control.}
We embed $3 \times 3$ games into $12 \times 12$ matrices by adding actions that are
irrelevant under iterated dominance. New rows have payoffs below the original minimum
(strictly dominated for the row player). On the original row support, new columns have
payoffs above the original maximum and are therefore dominated for the column player, who
minimizes the row player's payoff. The true Nash equilibrium is unchanged---only the
matrix dimension increases. As a negative control, we also place the same $3 \times 3$
block in the top-left corner of an otherwise random $N \times N$ game. If success came
from detecting a familiar sub-block or from output-format shortcuts, random padding
should also succeed; if dominated-action reasoning is the operative mechanism, only dominated padding
should preserve high success.

\begin{table}[t]
\centering
\caption{Single-seed controlled evidence for the $n \approx 12$ cliff under LoRA adaptation
via dominated-action padding and a random-padding negative control. Entries report $s@0.10$ over 50 games with 4
samples per game. Dominated padding preserves the original $3 \times 3$ equilibrium;
random padding embeds the same $3 \times 3$ block but changes the game.}
\label{tab:causal_cliff}
\vspace{0.5em}
\small
\begin{tabular}{llcccc}
\toprule
\textbf{Model} & \textbf{Condition} & $8{\times}8$ & $12{\times}12$ & $15{\times}15$ & $20{\times}20$ \\
\midrule
\multirow{3}{*}{Base} 
& Dense random & 0.00 & 0.00 & 0.00 & 0.00 \\
& Dominated $3{\to}N$ & 0.00 & 0.00 & 0.02 & 0.00 \\
& Random $3{\to}N$ & 0.00 & 0.00 & 0.00 & 0.00 \\
\midrule
\multirow{3}{*}{SFT}
& Dense random & 0.16 & 0.02 & 0.02 & 0.02 \\
& Dominated $3{\to}N$ & \textbf{1.00} & \textbf{1.00} & \textbf{0.98} & \textbf{0.98} \\
& Random $3{\to}N$ & 0.22 & 0.06 & 0.02 & 0.04 \\
\midrule
\multirow{3}{*}{VERGE}
& Dense random & 0.36 & 0.04 & 0.02 & 0.04 \\
& Dominated $3{\to}N$ & \textbf{0.94} & \textbf{0.90} & \textbf{0.90} & \textbf{0.76} \\
& Random $3{\to}N$ & 0.22 & 0.02 & 0.02 & 0.02 \\
\bottomrule
\end{tabular}
\end{table}

\cref{tab:causal_cliff} provides the strongest controlled evidence in the paper. Across three
SFT seeds, $3 \times 3$ games embedded in $12 \times 12$ matrices via iteratively
dominated padding achieve $s@0.10=0.95{\pm}.08$ and remain at $0.97{\pm}.01$ even
through $20 \times 20$ padding, while dense $12 \times 12$ games achieve
$0.01{\pm}.01$ and random-padded $12 \times 12$ games achieve only $0.06$ in the
negative-control seed. VERGE shows the same separation in the diagnostic seed:
$0.90$ on dominated $3{\to}12$ padding versus $0.04$ on dense $12 \times 12$ and
$0.02$ on random $3{\to}12$ padding, with gradual degradation to $0.76$ at
$20 \times 20$. The base model without fine-tuning fails on both dominated and
random-padded games, confirming that dominated-action recognition requires training.
The gap between dominated and random padding rules out token length
and embedded-block recognition as sufficient explanations for success. Because the
dominated and random controls have the same output dimensionality, the experiment also
shows that output dimension alone is not sufficient to explain the success cases; the
remaining failure on dense $12{\times}12$ games must involve strategic depth and
structured serialization together.

\begin{figure}[t]
\centering
\includegraphics[width=0.85\textwidth]{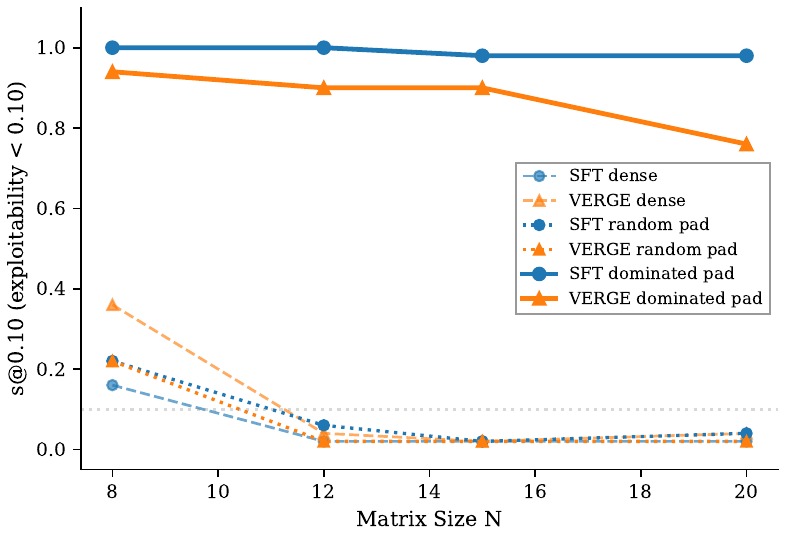}
\caption{Dominated-padding experiment: $s@0.10$ as a function of padded matrix size $N$
for $3{\times}3 \to N$ embeddings with iteratively dominated actions versus dense random
$N{\times}N$ games and random-padded negative controls. Dominated embeddings remain far
above both controls through $N{=}20$.}
\label{fig:dom_pad}
\end{figure}

\paragraph{Equivariance tests.}
To verify that the model learned genuine algorithmic behavior rather than position-dependent
heuristics, we test game-theoretic invariances. For random row/column permutations of the
payoff matrix, the mean reward difference is $0.019$--$0.054$ (near-zero); for random payoff
shifts and positive scaling, the error is $0.029$--$0.044$. The model is approximately
equivariant to the symmetries that any game-solving algorithm must respect.

\subsection{Curriculum and Serialization Tradeoff}
\label{sec:curriculum}

Training on larger games ($n \in \{3, 4, 5\}$ instead of $\{2, 3\}$) reveals a
surprising low-rank tradeoff. SFT-large averages $s@0.10 = 0.62$ at $7 \times 7$ over two seeds
(vs.\ $0.32$ for the three-seed SFT row on $\{2, 3\}$), showing improved
\emph{conditional} strategy quality.
However, at $12 \times 12$, the valid-output rate drops to 0\% (compared to 99\% for SFT on
$\{2, 3\}$), meaning the model cannot produce a parseable probability vector of length~12.

Scratchpad-output supervision separates the two failure modes. In 700-step $r=32$
controls with $N{=}30$ evaluation games, a $2$--$3$ curriculum reaches only
$0.19{\pm}.04$ at $7{\times}7$ and $0.07{\pm}.06$ at $12{\times}12$. Expanding the
curriculum to $2$--$5$ improves near-OOD success to $0.50{\pm}.06$ at $7{\times}7$,
but remains at $0.04{\pm}.04$ at $12{\times}12$ despite high $12{\times}12$ valid-output
rate ($0.93{\pm}.05$). Thus scratchpad supervision largely restores valid
12-dimensional serialization, but the resulting strategies remain exploitable.

This reveals two separable capabilities: \emph{strategic computation}
(selecting good probability weights) and \emph{structured output emission}
(serializing a valid $n$-dimensional distribution). A higher-rank control strengthens
the diagnosis: increasing LoRA from $r=32$ to $r=128$ improves near-OOD success
(final $s@0.10 = 1.00,0.98,0.80,0.62,0.42,0.26$ on $2$--$7$), but still yields
$0\%$ valid outputs and $s@0.10=0.00$ at $12 \times 12$. The cliff is therefore not
removed by scratchpad traces or a simple $4\times$ LoRA-rank increase; whether full
fine-tuning or explicit structured decoding resolves it remains open.

\subsection{VERGE Training Dynamics}
\label{sec:grpo_exp}

VERGE uses GRPO~\citep{shao2024deepseekmath} with exploitability reward, $G=8$ samples
per generated game and batch-level advantage normalization with 4-step gradient accumulation.
Training converges within 2000 steps on
$2{\times}2$/$3{\times}3$ games. We note that gradient accumulation is critical:
reducing to per-mini-batch updates (effective batch size 2) causes gradient norm spikes
up to $93\times$ baseline and training collapse within 500 steps (see~\cref{sec:exp_details}).
This motivates our choice of $G=8$ with 4-step accumulation throughout.

\section{Discussion and Conclusion}
\label{sec:conclusion}

Matrix games separate semantic lookup, learned approximate computation, and output-interface
failure: procedural training improves random $5{\times}5$ success from 2\% to 61\%, while
dominated-action padding isolates the $n{\approx}12$ cliff from token length. The study is
deliberately narrow---mostly Qwen3.5-9B LoRA on zero-sum matrices, with $N{=}30$--150
controlled evaluations and seed-sensitive VERGE runs---so the three-regime claim should be
read as a testbed result, not a universal claim about all strategic interaction. A
scratchpad prompt alone remains near failure at $12{\times}12$ ($s@0.10=0.04$, pass@1
$=0.00$), and scratchpad SFT improves $7{\times}7$ near-OOD success but leaves
$12{\times}12$ success at $0.04{\pm}.04$. Full fine-tuning, structured decoding,
variance-controlled residual RL, and general-sum or extensive-form games remain open.

\newpage
\label{page:references-start}

\bibliographystyle{plainnat}
\bibliography{references}

@inproceedings{garg2022can,
  author       = {Shivam Garg and Dimitris Tsipras and Percy Liang and Gregory Valiant},
  title        = {What Can Transformers Learn In-Context? {A} Case Study of Simple Function Classes},
  booktitle    = {NeurIPS},
  year         = {2022},
}

@inproceedings{le2022coderl,
  author       = {Hung Le and Yue Wang and Akhilesh Deepak Gotmare and Silvio Savarese and Steven Chu-Hong Hoi},
  title        = {CodeRL: Mastering Code Generation through Pretrained Models and Deep Reinforcement Learning},
  booktitle    = {NeurIPS},
  year         = {2022},
}

@inproceedings{yuan2024self,
  author       = {Weizhe Yuan and Richard Yuanzhe Pang and Kyunghyun Cho and Xian Li and Sainbayar Sukhbaatar and Jing Xu and Jason Weston},
  title        = {Self-Rewarding Language Models},
  booktitle    = {ICML},
  year         = {2024},
}

@inproceedings{munos2024nash,
  author       = {R{\'e}mi Munos and Michal Valko and Daniele Calandriello and Mohammad Gheshlaghi Azar and Mark Rowland and Daniel Guo and Yunhao Tang and Matthieu Geist and Thomas Mesnard and others},
  title        = {Nash Learning from Human Feedback},
  booktitle    = {ICML},
  year         = {2024},
}

@article{shao2024deepseekmath,
  author       = {Zhihong Shao and Peiyi Wang and Qihao Zhu and Runxin Xu and Junxiao Song and Xiao Bi and Haowei Zhang and Mingchuan Zhang and Y.K. Li and Y. Wu and Daya Guo},
  title        = {{DeepSeekMath}: Pushing the Limits of Mathematical Reasoning in Open Language Models},
  journal      = {arXiv preprint arXiv:2402.03300},
  year         = {2024},
}

@article{gandhi2023strategic,
  author       = {Kanishk Gandhi and Dorsa Sadigh and Noah D. Goodman},
  title        = {Strategic Reasoning with Language Models},
  journal      = {arXiv preprint arXiv:2305.19165},
  year         = {2023},
}

@inproceedings{duan2024gtbench,
  author       = {Jinhao Duan and Renming Zhang and James Diffenderfer and Bhavya Kailkhura and Lichao Sun and Elias Stengel-Eskin and Mohit Bansal and Tianlong Chen and Kaidi Xu},
  title        = {{GTBench}: Uncovering the Strategic Reasoning Capabilities of {LLMs} via Game-Theoretic Evaluations},
  booktitle    = {NeurIPS},
  year         = {2024},
}

@inproceedings{fan2024can,
  author       = {Caoyun Fan and Jindou Chen and Yaohui Jin and Hao He},
  title        = {Can Large Language Models Serve as Rational Players in Game Theory? {A} Systematic Analysis},
  booktitle    = {AAAI},
  year         = {2024},
}

@article{silva2024mixed,
  author       = {Alonso Silva},
  title        = {Large Language Models Playing Mixed Strategy Nash Equilibrium Games},
  journal      = {arXiv preprint arXiv:2406.10574},
  year         = {2024},
}

@article{wang2024tmgbench,
  author       = {Haochuan Wang and Xiachong Feng and Lei Li and Zhanyue Qin and Dianbo Sui and Lingpeng Kong},
  title        = {{TMGBench}: A Systematic Game Benchmark for Evaluating Strategic Reasoning Abilities of {LLMs}},
  journal      = {arXiv preprint arXiv:2410.10479},
  year         = {2024},
}

@inproceedings{anil2022exploring,
  author       = {Cem Anil and Yuhuai Wu and Anders Andreassen and Aitor Lewkowycz and Vedant Misra and Vinay Ramasesh and Ambrose Slone and Guy Gur-Ari and Ethan Dyer and Behnam Neyshabur},
  title        = {Exploring Length Generalization in Large Language Models},
  booktitle    = {NeurIPS},
  year         = {2022},
}

@inproceedings{chen2024self,
  author       = {Zixiang Chen and Yihe Deng and Huizhuo Yuan and Kaixuan Ji and Quanquan Gu},
  title        = {Self-Play Fine-Tuning Converts Weak Language Models to Strong Language Models},
  booktitle    = {ICML},
  year         = {2024},
}

@inproceedings{lanctot2017unified,
  author       = {Marc Lanctot and Vinicius Zambaldi and Audrunas Gruslys and Angeliki Lazaridou and Karl Tuyls and Julien P{\'e}rolat and David Silver and Thore Graepel},
  title        = {A Unified Game-Theoretic Approach to Multiagent Reinforcement Learning},
  booktitle    = {NeurIPS},
  year         = {2017},
}

@inproceedings{wu2024sppo,
  author       = {Yue Wu and Zhiqing Sun and Huizhuo Yuan and Kaixuan Ji and Yiming Yang and Quanquan Gu},
  title        = {Self-Play Preference Optimization for Language Model Alignment},
  booktitle    = {NeurIPS},
  year         = {2024},
}

@article{yuan2025marshal,
  author       = {Huining Yuan and Zelai Xu and Zheyue Tan and Xiangmin Yi and Mo Guang and Kaiwen Long and Haojia Hui and Boxun Li and Xinlei Chen and Bo Zhao and Xiao-Ping Zhang and Chao Yu and Yu Wang},
  title        = {{MARSHAL}: Incentivizing Multi-Agent Reasoning via Self-Play with Strategic {LLMs}},
  journal      = {arXiv preprint arXiv:2510.15414},
  year         = {2025},
}

@inproceedings{wei2022chain,
  author       = {Jason Wei and Xuezhi Wang and Dale Schuurmans and Maarten Bosma and Fei Xia and Ed Chi and Quoc V. Le and Denny Zhou},
  title        = {Chain-of-Thought Prompting Elicits Reasoning in Large Language Models},
  booktitle    = {NeurIPS},
  year         = {2022},
}

@inproceedings{kojima2022large,
  author       = {Takeshi Kojima and Shixiang Shane Gu and Machel Reid and Yutaka Matsuo and Yusuke Iwasawa},
  title        = {Large Language Models are Zero-Shot Reasoners},
  booktitle    = {NeurIPS},
  year         = {2022},
}

@article{nye2021show,
  author       = {Maxwell Nye and Anders Johan Andreassen and Guy Gur-Ari and Henryk Michalewski and Jacob Austin and David Bieber and David Dohan and Aitor Lewkowycz and Maarten Bosma and David Luan and Charles Sutton and Augustus Odena},
  title        = {Show Your Work: Scratchpads for Intermediate Computation with Language Models},
  journal      = {arXiv preprint arXiv:2112.00114},
  year         = {2021},
}

@article{cobbe2021training,
  author       = {Karl Cobbe and Vineet Kosaraju and Mohammad Bavarian and Mark Chen and Heewoo Jun and Lukasz Kaiser and Matthias Plappert and Jerry Tworek and Jacob Hilton and Reiichiro Nakano and Christopher Hesse and John Schulman},
  title        = {Training Verifiers to Solve Math Word Problems},
  journal      = {arXiv preprint arXiv:2110.14168},
  year         = {2021},
}

@inproceedings{christiano2017deep,
  author       = {Paul F. Christiano and Jan Leike and Tom B. Brown and Miljan Martic and Shane Legg and Dario Amodei},
  title        = {Deep Reinforcement Learning from Human Preferences},
  booktitle    = {NeurIPS},
  year         = {2017},
}

@inproceedings{zelikman2022star,
  author       = {Eric Zelikman and Yuhuai Wu and Jesse Mu and Noah D. Goodman},
  title        = {{STaR}: Bootstrapping Reasoning With Reasoning},
  booktitle    = {NeurIPS},
  year         = {2022},
}

@article{bakhtin2022human,
  author       = {Anton Bakhtin and Noam Brown and Emily Dinan and Gabriele Farina and Colin Flaherty and Daniel Fried and Andrew Goff and Jonathan Gray and Hengyuan Hu and Athul Paul Jacob and others},
  title        = {Human-Level Play in the Game of Diplomacy by Combining Language Models with Strategic Reasoning},
  journal      = {Science},
  year         = {2022},
}

@article{nash1951non,
  author       = {John Nash},
  title        = {Non-Cooperative Games},
  journal      = {Annals of Mathematics},
  year         = {1951},
}

@article{lemke1964equilibrium,
  author       = {Carlton E. Lemke and Joseph T. Howson Jr.},
  title        = {Equilibrium Points of Bimatrix Games},
  journal      = {Journal of the Society for Industrial and Applied Mathematics},
  year         = {1964},
}

@book{nisan2007algorithmic,
  author       = {Noam Nisan and Tim Roughgarden and {\'E}va Tardos and Vijay V. Vazirani},
  title        = {Algorithmic Game Theory},
  publisher    = {Cambridge University Press},
  year         = {2007},
}

@article{lightman2023lets,
  author       = {Hunter Lightman and Vineet Kosaraju and Yuri Burda and Harrison Edwards and Bowen Baker and Teddy Lee and Jan Leike and John Schulman and Ilya Sutskever and Karl Cobbe},
  title        = {Let's Verify Step by Step},
  journal      = {arXiv preprint arXiv:2305.20050},
  year         = {2023},
}

@article{guo2025deepseekr1,
  author       = {Daya Guo and Dejian Yang and Haowei Zhang and Junxiao Song and Ruoyu Zhang and Runxin Xu and Qihao Zhu and Shirong Ma and Peiyi Wang and Xiao Bi and others},
  title        = {{DeepSeek-R1}: Incentivizing Reasoning Capability in {LLMs} via Reinforcement Learning},
  journal      = {arXiv preprint arXiv:2501.12948},
  year         = {2025},
  eprint       = {2501.12948},
  archiveprefix = {arXiv},
  doi          = {10.48550/arXiv.2501.12948},
}


\appendix

\section{Proof of Residual Stability (Theorem~\ref{thm:residual_stability})}
\label{sec:proof_stability}

\begin{proof}
Fix $(\mathbf{p}, \mathbf{q}) \in \Delta^n \times \Delta^n$.
For the row-regret term:
\begin{align}
|\max_i (A\mathbf{q})_i - \max_i (B\mathbf{q})_i|
&\leq \max_i |(A\mathbf{q})_i - (B\mathbf{q})_i| \\
&= \max_i |((A-B)\mathbf{q})_i| \\
&\leq \max_i \sum_j |A_{ij} - B_{ij}| \cdot q_j \\
&\leq \|A - B\|_\infty \cdot \|\mathbf{q}\|_1 = \|A - B\|_\infty.
\end{align}
Expanding~\eqref{eq:exploitability}, the $\mathbf{p}^\top A \mathbf{q}$ terms cancel:
$\exploit(A, \mathbf{p}, \mathbf{q}) = \max_i (A\mathbf{q})_i - \min_j (\mathbf{p}^\top A)_j$.
Thus:
\begin{align}
&|\exploit(A, \mathbf{p}, \mathbf{q}) - \exploit(B, \mathbf{p}, \mathbf{q})| \nonumber \\
&\quad \leq |\max_i (A\mathbf{q})_i - \max_i (B\mathbf{q})_i|
  + |\min_j (\mathbf{p}^\top\! A)_j - \min_j (\mathbf{p}^\top\! B)_j| \\
&\quad \leq \|A - B\|_\infty + \|A - B\|_\infty = 2\|A - B\|_\infty.
\end{align}
The second term follows by the same argument applied to the rows of $(A-B)^\top$
with $\|\mathbf{p}\|_1 = 1$. No dimension factor $n$ appears because the simplex
constraints absorb it.

For selector instability, let $M = \begin{psmallmatrix} 1 & -1 \\ -1 & 1 \end{psmallmatrix}$
(matching pennies) and consider the scaled family $A_\epsilon = \epsilon \cdot M$.
For any $\epsilon > 0$, the unique Nash equilibrium is
$\mathbf{p}^* = \mathbf{q}^* = (1/2, 1/2)$, so $T(A_\epsilon) = ((1/2,1/2),(1/2,1/2))$.
At $\epsilon = 0$, the payoff matrix is all zeros, and \emph{every} mixed strategy is a
Nash equilibrium (exploitability is identically zero). Any vertex-returning selector
(including common deterministic LP tie-breaking rules) returns a pure strategy,
e.g.\ $T(A_0) = ((1,0),(1,0))$. Thus
$\|T(A_\epsilon) - T(A_0)\|_1 \geq 1$ while $\|A_\epsilon - A_0\|_\infty = \epsilon \to 0$,
confirming $\Omega(1)$ discontinuity for vertex-returning selectors.
\end{proof}

\section{GRPO Gradient Cancellation in Zero-Sum Self-Play}
\label{sec:proof_grpo}

This section proves that \emph{role-merged} GRPO---where the same generated output
is scored in both row and column roles within a single normalization group---yields
exactly zero advantage-weighted strategic gradient. The statement excludes KL
regularization, whose gradient can still pull the policy toward the reference but does
not provide a game-theoretic learning signal. This is a specific failure mode of the naive implementation
where one output serves both roles simultaneously; it does not apply to self-play
architectures with separate policy copies or asymmetric normalization groups.
The result motivates our use of cooperative exploitability rewards (\cref{sec:memorization}).

\begin{theorem}[GRPO gradient cancellation in symmetric zero-sum self-play]
\label{thm:grpo_nosignal}
Consider a zero-sum game where a shared policy $\pi_\theta$ plays both roles.
GRPO generates $G$ responses, each evaluated as both row (reward $r_i$) and column
(reward $-r_i$). Under role-merged group normalization, the advantage-weighted
policy-gradient component of GRPO is exactly $\mathbf{0}$ for every realization.
\end{theorem}

\begin{proof}
Under symmetric zero-sum self-play, each game instance $k$ produces a row reward $r_k$
and a column reward $-r_k$. The group of $2G$ rewards is
$\mathcal{R} = \{r_1, \ldots, r_G, -r_1, \ldots, -r_G\}$.

The group mean is $\bar{r} = \frac{1}{2G}\sum_{i=1}^G (r_i + (-r_i)) = 0$ identically
(not in expectation---exactly zero for every realization).

The group standard deviation is $\sigma = \sqrt{\frac{1}{2G}\sum_{i=1}^G (r_i^2 + r_i^2)}
= \sqrt{\frac{1}{G}\sum_{i=1}^G r_i^2}$.
If $\sigma = 0$, then all $r_i=0$ and all normalized advantages are zero, so the gradient
is zero. Otherwise $\sigma > 0$ and the normalized advantages below are well-defined.

The advantage for the row-role response $i$ is $\hat{A}_i^{\text{row}} = r_i / \sigma$,
and for the column-role response $i$ is $\hat{A}_i^{\text{col}} = -r_i / \sigma$.
Since each response~$k$ proposes a complete strategy pair and is evaluated in both roles,
the \emph{same} output~$y_k$ appears in both gradient terms. The GRPO gradient
contribution from response~$k$ is:
\begin{equation}
\hat{A}_k^{\text{row}} \nabla \log \pi_\theta(y_k) + \hat{A}_k^{\text{col}} \nabla \log \pi_\theta(y_k)
= \bigl(\hat{A}_k^{\text{row}} + \hat{A}_k^{\text{col}}\bigr) \nabla \log \pi_\theta(y_k)
= 0 \cdot \nabla \log \pi_\theta(y_k) = \mathbf{0}.
\end{equation}
Summing over all $G$ responses, the advantage-weighted strategic gradient is
$\mathbf{0}$---not in expectation, but exactly, for every realization. A KL
regularization term, if included, is outside this cancellation statement and contributes
only reference-policy pressure rather than a row/column strategic signal.
\end{proof}

\section{Additional Experimental Details}
\label{sec:exp_details}

\paragraph{LoRA configuration.}
All experiments use LoRA with $r=32$, $\alpha=64$, dropout $0.05$, applied to
$\{q, k, v, o, \text{gate}, \text{up}, \text{down}\}$ projections (7 target modules).
Total trainable parameters: 58.2M / 9.0B (0.65\%).

\paragraph{Optimizer and schedule.}
We use AdamW with weight decay $0.01$, linear warmup followed by cosine decay,
learning rates $5{\times}10^{-5}$ for SFT and $1{\times}10^{-5}$ for VERGE,
and VERGE KL coefficient $0.05$ with GRPO clipping range $0.2$.

\paragraph{Game generation details.}
Integer payoffs are drawn uniformly from $\{-9, \ldots, 9\}$, then normalized:
$A \leftarrow 2A / (\max A - \min A)$, yielding entries in approximately $[-2, 2]$.
If all entries are equal, one is perturbed by +1 to ensure non-degeneracy.

\paragraph{Evaluation protocol.}
For each game size $n$ and evaluation set (fixed seed=99):
\begin{itemize}[nosep]
\item Generate 50 random games (30 for intermediate checkpoints)
\item For each game, generate 4 samples at temperature 0.7
\item Parse JSON output; normalize to probability simplex
\item Report best-of-4 exploitability reward and first-sample pass@1
\item $s@0.10$ = fraction of games where best reward $> 0.90$
\end{itemize}

\paragraph{Computational cost.}
Per training seed, VERGE 2000 steps takes $\sim$14 hours on 1$\times$A800,
and SFT 2000 steps takes $\sim$4 hours on 1$\times$A800. We parallelize
independent seeds and evaluation jobs across available A800 GPUs. Full evaluation
suite: $\sim$2 hours per checkpoint.

\section{Limitations}
\label{sec:limitations}

This study is intentionally narrow. The main experiments use procedurally generated
two-player zero-sum matrix games, mostly Qwen3.5-9B with LoRA adaptation, and fixed-size
random evaluation sets rather than human or real-world strategic tasks. Cross-architecture
evidence is limited, full fine-tuning is not evaluated, and the structured-output interface
is left to the model rather than enforced by constrained decoding. VERGE should therefore
be read as a residual-reward diagnostic rather than a stable method that dominates solver
supervision. The conclusions identify controlled failure modes in this testbed; they do
not establish a universal law of strategic reasoning across model families, game classes,
or deployment settings.

\section{Broader Impact}
\label{sec:broader_impact}

This work is a synthetic evaluation and training study, not a deployed strategic agent.
Positive impacts include sharper tests for whether language models reason procedurally or
retrieve memorized solutions, which can improve evaluation practice for planning and
decision-support systems. Negative impacts are possible if stronger strategic-reasoning
agents are used in automated negotiation, persuasion, or planning settings without
oversight. To limit release risk, the supplemental material provides code and synthetic
game generators rather than trained agent checkpoints or a user-facing decision system.


\end{document}